\documentclass[journal]{IEEEtran}
\usepackage{tabularx}
\usepackage{soul}
\usepackage{url}
\usepackage[small]{caption}
\usepackage{graphicx}
\usepackage{float}
\usepackage{amsmath}
\usepackage{booktabs}
\usepackage{subfigure}
\usepackage{algorithm}
\usepackage{algorithmic}
\usepackage{amssymb}
\usepackage{amsmath}
\usepackage{multirow}
\usepackage{comment}
\usepackage{makecell}
\usepackage{bm}
\usepackage{color}
\usepackage{amsthm}

\newtheorem{theorem}{Theorem}

\begin{document}
\title{AGNN: Alternating Graph-Regularized Neural Networks to Alleviate Over-Smoothing}

\author{Zhaoliang Chen,
        Zhihao Wu,
        Zhenghong Lin,
        Shiping Wang,
        Claudia Plant and
        Wenzhong Guo
\thanks{This work is in part supported by the National Natural Science Foundation of China (Grant No. U21A20472), the Natural Science Foundation of Fujian Province (Grant No. 2020J01130193). Corresponding author: Wenzhong Guo.}
\thanks{Zhaoliang Chen, Zhihao Wu, Zhenghong Lin, Shiping Wang and Wenzhong Guo are with the College of Computer and Data Science, Fuzhou University, Fuzhou 350116, China and also with the Fujian Provincial Key Laboratory of Network Computing and Intelligent Information Processing, Fuzhou University, Fuzhou 350116, China (email: chenzl23@outlook.com, zhihaowu1999@gmail.com, hongzhenglin970323@gmail.com, shipingwangphd@163.com, guowenzhong@fzu.edu.cn).

Claudia Plant is with the Faculty of Computer Science and with the research network Data Science @ Uni Vienna, University of Vienna, 1090 Vienna, Austria (email: claudia.plant@univie.ac.at).
}
}

\markboth{IEEE Transactions on Neural Networks and Learning Systems}%
{Shell \MakeLowercase{\textit{et al.}}: Bare Demo of IEEEtran.cls for IEEE Journals}

\maketitle

\begin{abstract}
    Graph Convolutional Network (GCN) with the powerful capacity to explore graph-structural data has gained noticeable success in recent years.
    Nonetheless, most of the existing GCN-based models suffer from the notorious over-smoothing issue, owing to which shallow networks are extensively adopted.
    This may be problematic for complex graph datasets because a deeper GCN should be beneficial to propagating information across remote neighbors.
    Recent works have devoted effort to addressing over-smoothing problems, including establishing residual connection structure or fusing predictions from multi-layer models. 
    Because of the indistinguishable embeddings from deep layers, it is reasonable to generate more reliable predictions before conducting the combination of outputs from various layers.
    In light of this, we propose an Alternating Graph-regularized Neural Network (AGNN) composed of Graph Convolutional Layer (GCL) and Graph Embedding Layer (GEL).
    GEL is derived from the graph-regularized optimization containing Laplacian embedding term, which can alleviate the over-smoothing problem by periodic projection from the low-order feature space onto the high-order space.
    With more distinguishable features of distinct layers, an improved Adaboost strategy is utilized to aggregate outputs from each layer, which explores integrated embeddings of multi-hop neighbors.
    The proposed model is evaluated via a large number of experiments including performance comparison with some multi-layer or multi-order graph neural networks, which reveals the superior performance improvement of AGNN compared with state-of-the-art models.
\end{abstract}

\IEEEpeerreviewmaketitle

\section{Introduction}
Graph Neural Network (GNN) has become one of the promising technologies manipulating graph-structural data in recent years, 
obtaining remarkable achievement in various pattern recognition fields, including node classification or clustering \cite{zhang2021shne,ZhongW0HDNL021,li2020graph}, recommender systems \cite{YanhuiHybrid2022,wang2021dualgnn,deng2022g} and computer vision \cite{YangLLZWL20,XuWYHS21,XuHQXHH21}.
As one of the typical GNN-based models, Graph Convolutional Network (GCN) is receiving plentiful attention from a population of researchers \cite{KipfW17,BaiCJRH22}.
Owing to its powerful ability to extract knowledge from sparse weighted networks, GCN has also been adopted to weight prediction for sparse weighted graphs, such as dynamic graphs \cite{ShangYLZ22,9647958,9839318}.
Originated from GCN, Graph AutoEncoder (GAE) was also investigated to conduct weighted link predictions via the reconstruction of the adjacency matrix \cite{9889163,chen2020general,CuiZY020}.
GCN propagates node representations across topology networks via convolution operators on non-Euclidean space, which integrates node features and relationships involved in a graph.
Nonetheless, recent practice and theoretical analysis have indicated that a 2-layer GCN generally performs the best, and a deep GCN often leads to unfavorable performance, which is summarized into the over-smoothing issue.

\begin{figure}[!tbp]
    \centering
    \includegraphics[width=0.48\textwidth]{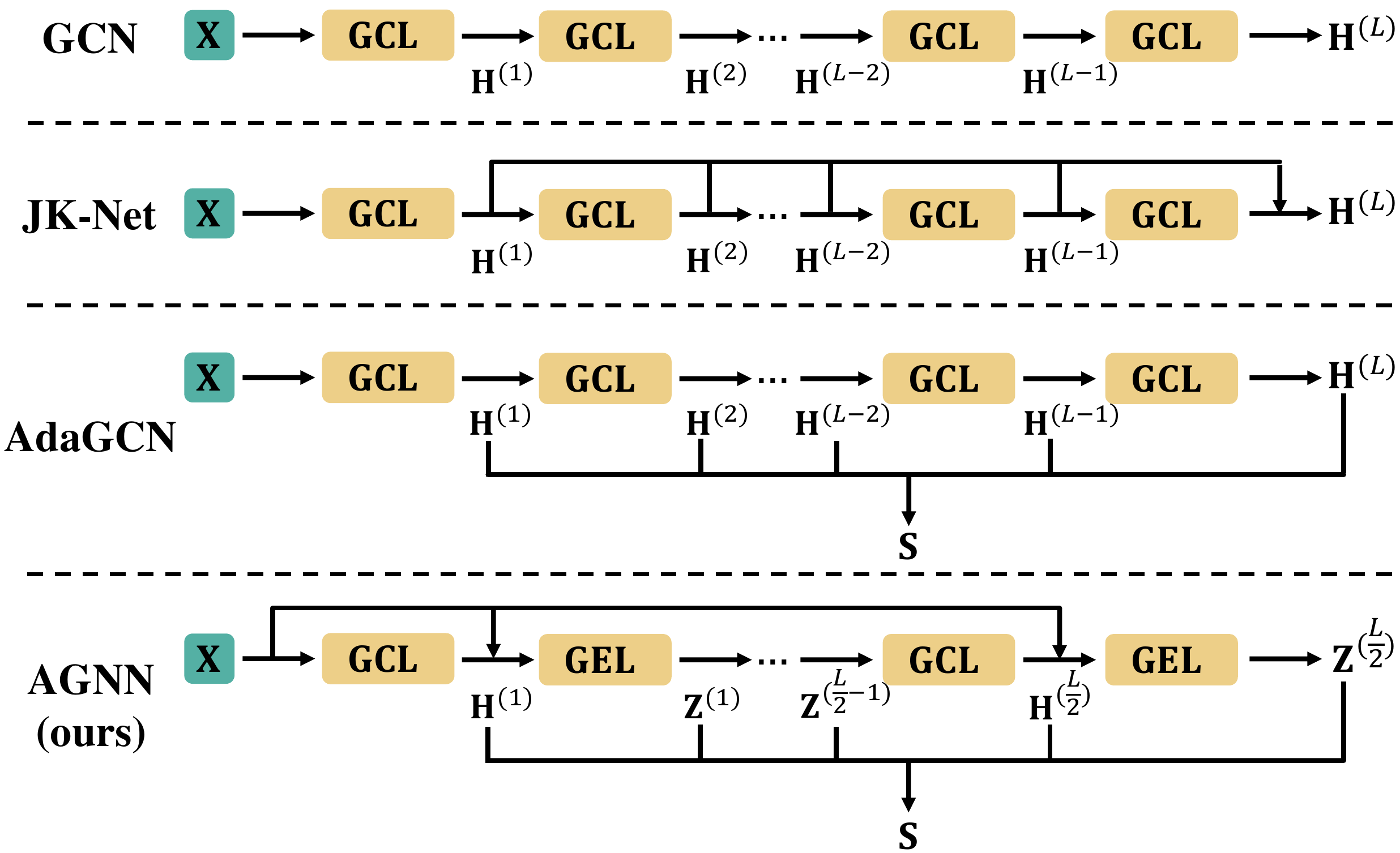}
    \caption{Architectures of numerous GCN-based methods and the proposed AGNN, where GCL is Graph Convolutional Layer and GEL is the proposed Graph Embedding Layer.
    GCN \cite{MinWW20} is a sequence of GCL, which encounters severe over-smoothing with deep layers. JK-Net \cite{XuLTSKJ18} adds connections among layers to carry all low-order information to the last layer. AdaGCN \cite{SunZL21} aggregates multi-hop embeddings of all layers. The proposed AGNN simultaneously carries low-order features to deep layers and accumulates node predictions from all layers.
    }
    \label{MethodCompare}
\end{figure}

Over-smoothing is a widely concerned deficiency of GCN, which has been extensively investigated.
Recent studies have proved that a graph convolution is exactly a special form of Laplacian smoothing, attributed to which a deeper GCN may result in indistinguishable node features and make the downstream classification tasks challenging \cite{LiHW18,eliasof2021pde,ChenZXMLZJH21}.
It makes most existing GCN-based models shallow and lack the ability to mine knowledge from high-order neighbors, which is more severe for datasets with high-degree nodes.
Considerable works have been devoted to solving this problem.
On one hand, some research attempted to consider a similar structure of residual connection leveraged in Euclidean deep convolutional networks \cite{XuLTSKJ18,ChenWHDL20,Li0TG19}.
Most of these methods made full use of embeddings from the previous layers or input matrix to avoid information loss.
On the other hand, some studies placed more emphases on effective exploration and combination of hidden representations from different hops of neighbors \cite{SunZL21,ChenZXMLZJH21,KlicperaBG19}.
A summary of the comparison between the representative algorithms (GCN \cite{KipfW17}, JK-Net \cite{XuLTSKJ18}, AdaGCN \cite{SunZL21}) and the proposed method in this paper are shown in Figure \ref{MethodCompare}.
Although some works have succeeded in relieving over-smoothing problems, they were still outperformed by a classical 2-layer GCN.
In addition, a direct linear combination of embeddings from hidden layers may not work effectively, because the similar and indistinguishable features from deeper layers can annihilate useful information from shallow layers and confound the predictions of classifiers.
Accordingly, it is crucial to develop a reliable network where each layer can yield accurate and distinguishable outputs before conducting the prediction fusion.

In pursuit of addressing the aforementioned problems, in this paper, we design an Alternating Graph-regularized Neural Network (AGNN) that enables the construction of deep layer architecture.
AGNN alternately performs forward computation of Graph Convolutional Layer (GCL) and Graph Embedding Layer (GEL).
In order to get rid of similar and indistinguishable features caused by over-smoothing, GEL is designed to project original node embeddings onto low-dimensional space in deep layers and preserve critical features via sparse outputs.
Thus, each proposed GEL aims to learn Laplacian-constrained sparse representations from original features, on the basis of the optimization problem w.r.t. the Laplacian-based graph regularization and sparsity constraint.
We derive the updating rules of this optimization target and transform them into GEL that preserves discriminative node embeddings during network training and alleviates the over-smoothing problem.
We analyze the network architecture and draw a conclusion that both GCL and GEL can be approximately regarded as solutions to distinct graph regularization problems.
Furthermore, with more accurate predictions yielded by GCL and GEL, an improved Adaboost algorithm is adopted to aggregate node representations from varying hidden layers, so that multi-order information from different depths of networks can be leveraged.
In summary, the contributions of this paper primarily lie in:

1) According to a graph-regularized optimization problem and its iterative solutions, we construct a new layer dubbed GEL, which can alleviate over-smoothing phenomenon via carrying low-order information to deep layers.

2) A graph-regularized neural network with alternating GCLs and GELs is proposed, which adopts both residual connection and embedding aggregation architecture. Its layers can be regarded as approximations of different graph optimization problems, which promote the interpretability of the model.

3) With more accurate embeddings yielded by deep layers, an improved Adaboost algorithm is designed to leverage features from distinct hidden layers, enabling the model to aggregate high-quality node representations from multi-hop neighbor propagation.

4) Substantial experimental results reveal the superiority of the proposed AGNN, which succeeds in coping with over-smoothing issue and outperforms the widely applied 2-layer GCN and other multi-layer GCN-based methods with deep network structures.

The rest contents of this paper are organized as follows.
Recent works of GCN and approaches to cope with the over-smoothing issues are discussed in Section \ref{relatedwork}.
In Section \ref{AGNN}, we elaborate on the proposed framework, including detailed analysis and comparison between AGNN and other models.
We evaluate AGNN with comprehensive experiments in Section \ref{Experiments}, looking into the performance under varying experimental settings.
Eventually, we conclude our works in Section \ref{Conclusion}.

\section{Related Works}\label{relatedwork}
\subsection{Graph Convolutional Network}
GCN has been applied to a multitude of applications and attracted attention from a wide range of researchers in recent years.
Xu et al. came up with a deep feature aggregation model with a graph convolutional network to conduct high spatial resolution scene classification \cite{xu2021deep}.
A GCN-based approach under the autoencoder framework was proposed to perform unsupervised community detection \cite{HeSJ0ZYZ20}.
In order to reduce the computational cost of graph convolutions, a low-pass collaborative filter was proposed to utilize GCN with a large graph \cite{YuQ20}. 
Gan et al. designed a multi-graph fusion model that combined the local graph and the global graph to produce a high-quality graph for GCN \cite{Gan2022Multigraph}.
An aggregation scheme was applied to promote the robustness of GCN against structural attacks \cite{ChenLPLZY21}.
Geometric scattering transformations and residual convolutions were leveraged to enhance the conventional GCN \cite{MinWW20}.
Xu et al. presented a spatiotemporal multi-graph convolutional fusion network, which exploited the graph-structural road network for urban vehicle emission estimation \cite{xu2020spatiotemporal}.
GCN with a question-aware gating mechanism was presented to aggregate evidences on the path-based graph \cite{TangSMXYL20}.
A new graph convolution operator was proposed to obtain robust embeddings in the spectral domain \cite{0002C0S21}.
The variant of GCN was derived via a modified Markov diffusion kernel, which explored the global and local contexts of nodes \cite{ZhuK21}.
Weighted link prediction is also a critical application of GCN.
For example, a dynamic GCN was proposed with a tensor M-product technique, to cope with adjacency tensor and feature tensor yielded from dynamic graphs \cite{MalikUHKA21}.
Cui et al. proposed an adaptive graph encoder to strengthen the filtered features for more discriminative node embeddings, which was applied to link prediction tasks \cite{CuiZY020}.
Wang et al. designed a temporal GAE, which encoded the fundamentally asymmetric nature of a directed network from neighborhood aggregation and captured link weights via reshaping the adjacency matrix \cite{9889163}.
However, most of these GCN-based models suffer from shallow network structure owing to the over-smoothing issue.

\subsection{Over-smoothing Issue}
Numerous works have investigated approaches to alleviate the over-smoothing issue.
An improved normalization trick applying the ``diagonal enhancement" was introduced to help build a deep GCN \cite{ChiangLSLBH19}.
Simple graph convolution \cite{WuSZFYW19} was proposed to mine high-order embeddings in the graph via utilizing the $k$-th power of the graph convolutional matrix and removing the ReLU function.
A multi-layer GCN was constructed with AdaBoost to linearly combine embeddings from varying layers \cite{SunZL21}.
Cui et al. restricted over-smoothing by extracting hierarchical multi-scale node feature representations \cite{cui2021learning}.
PPNP and APPNP \cite{KlicperaBG19} were presented to replace the power of the graph convolutional matrix inspired by the personalized PageRank matrix.
Residual connections and dilated convolutions in CNN were applied to promote the training of a deep GCN model.
Jumping knowledge networks preserved the locality of node embeddings via dense skip connections that merged features from each layer \cite{XuLTSKJ18}.
A deep GCN was proposed with residual connection and identity mapping to relieve the over-smoothing problem \cite{ChenWHDL20}.
Most of these methods attempted to alleviate over-smoothing via connecting distinct network layers, simplifying multi-order graph convolutions, or conducting multi-layer feature fusion.
Nonetheless, these existing works did not simultaneously consider cross-layer feature connection and aggregation of embeddings from varying layers,
which benefits a multi-layer model to obtain a more precise prediction.

\begin{figure*}[!htbp]
    \centering
    \includegraphics[width=\textwidth]{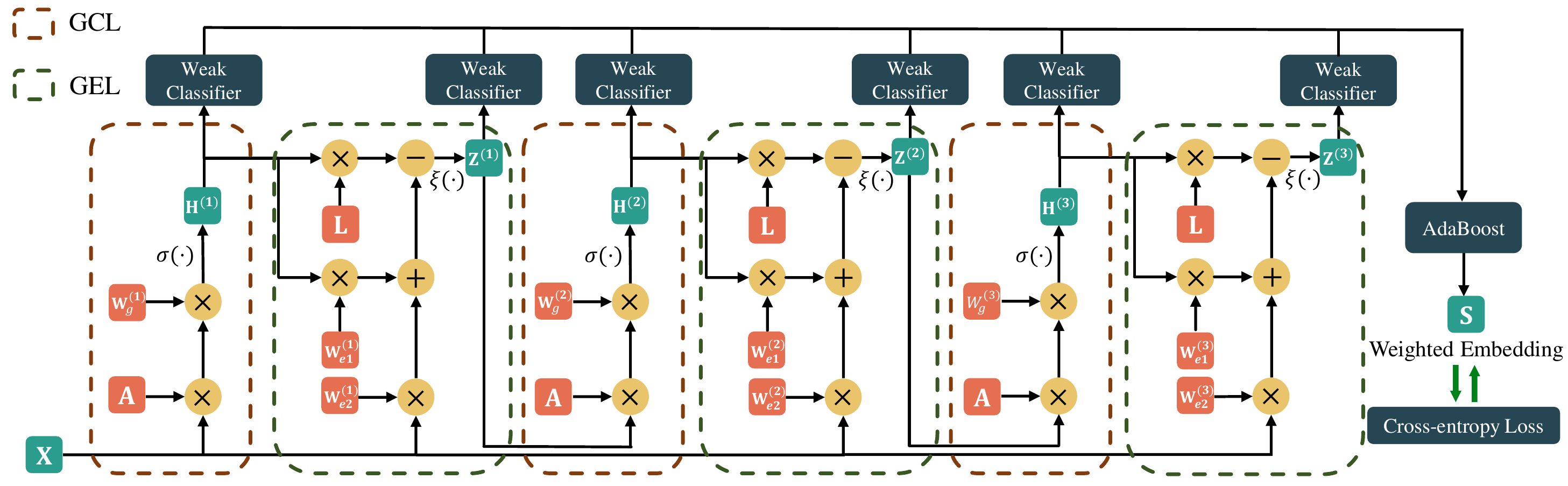}
    \caption{The framework of a 6-layer AGNN, which consists of three GCLs and three GELs. AGNN is a block-wise graph neural network framework constructed with alternating GCL and GEL, where each block contains a GCL and a GEL. For the purpose of exploiting reliable and discriminative multi-hop information, an improved AdaBoost strategy is utilized to aggregate node predictions yielded by weak classifiers in all layers, and the whole framework is evaluated by cross-entropy loss.}
    \label{Framework}     
\end{figure*}

\section{The Proposed Method}\label{AGNN}
Given a connected undirected graph $\mathcal{G} = (\mathcal{V}, \mathcal{E})$ with $n$ nodes and $e$ edges, we define the corresponding adjacency matrix as $\mathbf{A} \in \mathbb{R}^{n \times n}$.
The node features are denoted by the matrix $\mathbf{X} \in \mathbb{R}^{n \times m}$, i.e., $\mathbf{x}_{i}$ is an $m$-dimensional feature vector of the $i$-th node.
The proposed AGNN aims to carry out the semi-supervised classification task with the given set $\Omega$ of partially labeled samples and its corresponding ground truth matrix $\mathbf{Y} \in \mathbb{R}^{n \times c}$ encoding one-hot vectors, where $c$ is the number of classes.
For the purpose of better readability, we summarize the primarily used mathematical notations in Table \ref{Notations}.
As described in Figure \ref{Framework},
AGNN is a sequence of alternating GCL and GEL, and an improved AdaBoost strategy is adopted to merge multi-layer features.
Both GCL and GEL are constructed from graph-regularized optimization problems, which form a basic network block of AGNN.
In particular, GEL periodically projects the original node embeddings onto deep layers to alleviate over-smoothing, which introduces residual connections into AGNN.
In Section \ref{AGNN1}, we first analyze two distinct graph-regularized optimization problems, on the basis of which AGNN is constructed.
After that, an improved AdaBoost is designed to conduct multi-layer feature fusion in Section \ref{AGNN2}.
Finally, we summarized and analyzed the proposed model in Section \ref{AGNN3}, including time complexity analysis and comparison to related works.

\begin{table}[!htbp]
    \center
    \caption{A summary of primary notations in this paper.}
    \label{Notations}
    \begin{tabular}{l|l}
    \toprule
     Notations &       Explanations  \\
     \midrule
     $\mathbf{X}$             &  Feature matrix of nodes.  \\
     $\mathbf{A}$     &   Adjacency matrix.\\
     $\mathbf{Y}$     &   Label information.\\
     $\mathbf{H}^{(l)}$     &  Output of the $l$-th GCL.\\
     $\mathbf{Z}^{(l)}$     &  Output of the $l$-th GEL.\\
     $\mathbf{D}$     &   Diagonal degree matrix.\\
     $\mathbf{L}$     &  Laplacian matrix.\\
     $\mathbf{W}_{g}^{(l)}$     &  Weight matrix for the $l$-th GCL.\\
     $\mathbf{W}_{e1}^{(l)}$, $\mathbf{W}_{e2}^{(l)}$     &  Weight matrices for the $l$-th GEL.\\
     $\mathbf{Prox}_{g} (\cdot)$   & Proximal operator.\\
     $\xi_{(\mathbf{\theta}_{1}, \mathbf{\theta}_{2})}$   & MSReLU function with hyperparameters $\theta_{1}$ and $\theta_{2}$.\\
     $c (\cdot)$   & Weak classifier.\\
     $\mathbf{S}$  & Weighted embedding of multi-order feature fusion.\\
     $\alpha^{(l)}, \beta^{(l)}$ & Weights of classifiers for GCL and GEL.\\
     $\pi_{i}$  & Node weights for AdaBoost.\\
     $e_{\mathbf{H}}^{(l)}$, $e_{\mathbf{Z}}^{(l)}$  & Weighted classification error rates.\\
     $\eta_{i}$  & Node weight updating weight for AdaBoost.\\
     $R$  & Number of classes.\\
    \bottomrule
    \end{tabular}
\end{table}

\subsection{Alternating Graph Convolutional Layers and Graph Embedding Layers}
\label{AGNN1}
First, we revisit the definition of a vanilla graph convolution operator.
A GCL is formulated as
\begin{gather}\label{GCN}
    \begin{split}
    \mathbf{H}^{(l)} = \sigma \left( \tilde{\mathbf{D}}^{-\frac{1}{2}} \tilde{\mathbf{A}} \tilde{\mathbf{D}}^{-\frac{1}{2}} \mathbf{H}^{(l-1)} \mathbf{W}_{g}^{(l)} \right),
    \end{split}
\end{gather}
where $\tilde{\mathbf{A}} = \mathbf{A} + \mathbf{I}$ is the adjacency matrix that adds self-loop and $[\tilde{\mathbf{D}}]_{ii} = \sum_{j} [\tilde{\mathbf{A}}]_{ij}$ denotes the diagonal degree matrix.
The optional activation function is denoted as $\sigma(\cdot)$.
In fact, the added self-loop $\mathbf{A} + \mathbf{I}$ can be regarded as a simple residual connection to the previous layer.
Actually, GCL can be formulated as a graph-regularized optimization problem. Namely, we have the following theorem.
\begin{theorem}\label{theorem1}
    With a linear transformation matrix $\mathbf{W}_{g}^{(l)}$ and the node embedding $\mathbf{H}^{(l-1)}$ from the previous layer, the $l$-th GCL defined in Eq. \eqref{GCN} is the first-order approximation of the following optimization problem:
    \begin{equation}\label{OptimizationGCL}
        \mathbf{H}^{(l)} = \arg \min_{\mathbf{E}^{(l)}} \left \| \mathbf{E}^{(l)} - \mathbf{H}^{(l-1)} \mathbf{W}_{g}^{(l)} \right \|^{2}_{F} + \text{Tr} \left({\mathbf{E}^{(l)}}^{T} \tilde{\mathbf{L}} \mathbf{E}^{(l)} \right),
    \end{equation}
    where $\tilde{\mathbf{L}} = \mathbf{I} - \tilde{\mathbf{D}}^{-\frac{1}{2}} \tilde{\mathbf{A}} \tilde{\mathbf{D}}^{-\frac{1}{2}}$.
\end{theorem}
\begin{proof}
    The derivative w.r.t. $\mathbf{E}^{(l)}$ of the optimization problem defined in Eq. \eqref{OptimizationGCL} is
    \begin{gather}
        \frac{\partial \mathcal{J}}{\partial \mathbf{E}^{(l)}} = 2 \left(\mathbf{E}^{(l)} - \mathbf{H}^{(l-1)} \mathbf{W}_{g}^{(l)} \right) + 2 \tilde{\mathbf{L}} \mathbf{E}^{(l)}.
    \end{gather}
    Setting the derivative to 0, we have the closed-form solution
    \begin{gather}
        \mathbf{E}^{(l)} = \left(\mathbf{I} + \tilde{\mathbf{L}} \right)^{-1} \mathbf{H}^{(l-1)} \mathbf{W}_{g}^{(l)}.
    \end{gather}
    Because the term $\left(\mathbf{I} + \tilde{\mathbf{L}} \right)^{-1}$ can be decomposed into Taylor series, i.e.,
    \begin{gather}
        \left(\mathbf{I} + \tilde{\mathbf{L}} \right)^{-1} = \mathbf{I} - \tilde{\mathbf{L}} + \tilde{\mathbf{L}}^{2} + \ldots + (-1)^{t} \tilde{\mathbf{L}}^{t},
    \end{gather}
    we have the first-order truncated approximation as
    \begin{gather}
        \left(\mathbf{I} + \tilde{\mathbf{L}} \right)^{-1} \approx \mathbf{I} - \tilde{\mathbf{L}} = \tilde{\mathbf{D}}^{-\frac{1}{2}} \tilde{\mathbf{A}} \tilde{\mathbf{D}}^{-\frac{1}{2}}.
    \end{gather}
    Consequently, we obtain the approximation of $\mathbf{H}^{(l)}$ as 
    \begin{gather}
        \mathbf{H}^{(l)} \approx \tilde{\mathbf{D}}^{-\frac{1}{2}} \tilde{\mathbf{A}} \tilde{\mathbf{D}}^{-\frac{1}{2}}    \mathbf{H}^{(l-1)} \mathbf{W}_{g}^{(l)},
    \end{gather}
    which indicates that GCL is a first-order approximation of Problem \eqref{OptimizationGCL}.
\end{proof}

However, as we have analyzed before, deep graph convolutions often suffer from extremely indistinguishable features due to the over-smoothing phenomenon.
A solution is enabling the model to carry low-order information by connecting initial node features to each GCL.
Thus, we develop a new layer to bring features from the original space to deep layers.
To be consistent with GCL, we define a graph-regularized optimization problem to formulate this layer.
Instead of directly adding initial embeddings to the end of each GCL, a trainable projection derived from graph regularization optimization is applied, which adaptively learns low-dimensional representations from original node embeddings.
Namely, with graph embedding $\mathbf{H}$, we consider the following sparsity-constrained optimization
\begin{equation}\label{GraphEmbedding}
    \mathbf{Z}^{(l)} = \arg \min_{\mathbf{H}} \left \| \mathbf{X} - \mathbf{H} \mathbf{P}^{(l)} \right\|^{2}_{F}  
    +  \text{Tr} \left( {\mathbf{H}}^T \tilde{\mathbf{L}} \mathbf{H} \right)  + \left\| \mathbf{H} \right\|_{1},
\end{equation}
which explores Laplacian-constrained representations from the original feature space after the $l$-th GCL.
In pursuit of obtaining more distinguishable compressed node embeddings, we adopt $\left\| \mathbf{H} \right\|_{1}$ to consider sparse representations.
The sparsity constraint enables GEL to yield more discriminative node representations that only include important features, and alleviates the similar features of different nodes in deep layers, which is beneficial to solve the over-smoothing issue.
Consequently, it should have the same dimension as the previous GCL, and we can project it onto the original feature space with an over-complete dictionary matrix $\mathbf{P}^{(l)} \in \mathbb{R}^{d_{l} \times m}$, where $d_{l} < m$ is the number of hidden units at the $l$-th GCL.
In addition, we adopt the Laplacian embedding criterion $\text{Tr} \left( {\mathbf{H}}^T \tilde{\mathbf{L}} \mathbf{H} \right)$ to make nodes close when they are connected, where the Laplacian matrix $\tilde{\mathbf{L}}$ is precomputed.
In order to obtain more representative low-dimensional features,
$\left\| \mathbf{H} \right\|_{1}$ promoting the sparsity of outputs is added to extract robust projected embeddings during training.
Letting $f \left(\mathbf{H} \right) = \text{Tr} \left( {\mathbf{H}}^T \tilde{\mathbf{L}} \mathbf{H} \right) +  \left \| \mathbf{X} - \mathbf{H} \mathbf{P}^{(l)} \right\|^{2}_{F}$ and $g\left(\mathbf{H} \right) = \left\| \mathbf{H} \right\|_{1}$, we can derive the updating rules of Problem \eqref{GraphEmbedding} at $\mathbf{H}^{(l)}$ via proximal gradient descent method. Namely, 
\begin{gather}\label{PGD}
    \begin{split}
    \mathbf{Z}^{(l)} &= \arg \min_{\mathbf{H}} f(\mathbf{H}^{(l)}) + \langle\nabla f(\mathbf{H}^{(l)}), \mathbf{H}-\mathbf{H}^{(l)} \rangle \\
    &+ \frac{\tau }{2} \left\|\mathbf{H} - \mathbf{H}^{(l)} \right\|_{F}^{2} + \left\| \mathbf{H}^{(l)} \right\|_{1}\\
    &= \arg \min_{\mathbf{H}} \frac{\tau}{2}\left\|\mathbf{H} - \mathbf{Y} \right\|_{F}^{2} + \left\| \mathbf{H}^{(l)} \right\|_{1},
    \end{split}
\end{gather}
where $\mathbf{Y} = \mathbf{H}^{(l)} -  \frac{1}{\tau} \nabla f \left(\mathbf{H}^{(l)} \right)$, and $\tau$ is the Lipschitz constant.
Given the proximal operator $\mathbf{Prox}_{g} (\cdot)$,
Problem \eqref{PGD} can be solved by the proximal mapping w.r.t. $\ell_{1}$ norm.
Because we have the derivatives
\begin{gather}\label{DeltaF}
    \begin{split}
        \nabla f(\mathbf{H}^{(l)}) = 2 \tilde{\mathbf{L}} \mathbf{H}^{(l)} + 2 \left( \mathbf{H}^{(l)} \mathbf{P}^{(l)} - \mathbf{X}  \right) {\mathbf{P}^{(l)}}^{T},
    \end{split}
\end{gather}
the proximal mapping can be derived from
\begin{gather}\label{proximalMapping2}
    \begin{split}
        &\mathbf{Z}^{(l)} = \mathbf{Prox}_{g} \left( \mathbf{H}^{(l)} -  \frac{1}{\tau} \nabla f(\mathbf{H}^{(l)}) \right) \\
                         &= \mathbf{Prox}_{g} \left( \mathbf{H}^{(l)} -  \frac{1}{\tau} \left( 2 \tilde{\mathbf{L}} \mathbf{H}^{(l)} + 2 \left( \mathbf{H}^{(l)} \mathbf{P}^{(l)} - \mathbf{X}  \right) {\mathbf{P}^{(l)}}^{T} \right) \right) \\
                         &= \mathbf{Prox}_{g} \left(\mathbf{H}^{(l)} \left( \mathbf{I} - \frac{2}{\tau} \mathbf{P}^{(l)} {\mathbf{P}^{(l)}}^{T} \right) - \frac{2}{\tau} \tilde{\mathbf{L}} \mathbf{H}^{(l)} + \frac{2}{\tau} \mathbf{X} {\mathbf{P}^{(l)}}^{T} \right).
    \end{split}
\end{gather}
Transforming terms $\mathbf{I} - \frac{2}{\tau} \mathbf{P}^{(l)} {\mathbf{P}^{(l)}}^{T}$ and $\frac{2}{\tau} {\mathbf{P}^{(l)}}^{T}$ into trainable weight matrices $\mathbf{W}_{e1}^{(l)} \in \mathbb{R}^{d_{l} \times d_{l}}$ and $\mathbf{W}_{e2}^{(l)} \in \mathbb{R}^{m \times d_{l}}$ respectively, we have the following proximal projection
\begin{gather}\label{proximalMapping3}
    \begin{split}
        \mathbf{Z}^{(l)} = \mathbf{Prox}_{g} \left(\mathbf{H}^{(l)} \mathbf{W}_{e1}^{(l)} + \mathbf{X} \mathbf{W}_{e2}^{(l)} - \lambda \tilde{\mathbf{L}} \mathbf{H}^{(l)}  \right),
    \end{split}
\end{gather}
where $\lambda = \frac{2}{\tau}$ is a hyperparameter.
Because $\mathbf{Prox}_{g}(\cdot)$ can be regarded as an activation function,
Eq. \eqref{proximalMapping3} is similar to the definition of a neural network layer with two trainable weight matrices.
In particular, the proximal operator for $\ell_{1}$ constraint promoting the sparsity is
\begin{gather}
    \mathbf{Prox}_{g} \left(\mathbf{Z}^{(l)}_{ij} \right) = \text{sign} \left(\mathbf{Z}^{(l)}_{ij} \right) \left( \left|\mathbf{Z}^{(l)}_{ij} \right| - \theta \right)_{+},
\end{gather}
which is the Soft Thresholding (ST) function and $\theta$ is the hyperparameter to guarantee the sparsity of the output \cite{DBLP:conf/icml/GregorL10}.
It can be realized by a parameterized ReLU-based activation function, i.e.,
\begin{gather}
\xi_{\theta} (z) = \text{ReLU} \left(z - \theta \right) - \text{ReLU} \left(- z - \theta \right).
\end{gather}
Due to the definition of the ST function, $\xi_{\theta} (z)$ is actually smaller than $|z|$ when $z > \theta$ and $z < -\theta$.
This may be problematic due to the gap between original features and outputs of $\xi_{\theta} (z)$ when $\theta$ is relatively large.
For the sake of relieving the influence of this problem,
in this paper, we adopt a multi-stage proximal projection for the sparsity constraint, as shown below:
\begin{align}\label{ActivationFunction}
    \xi_{(\mathbf{\theta}_{1}, \mathbf{\theta}_{2})}(z) =
    \left\{\begin{matrix}
    z, & \theta_{2} \leq z, \\
    (\frac{2\theta_{2} - \theta_{1}}{\theta_{2}})(z - \theta_{1}),           & \theta_{1} \leq z < \theta_{2}, \\
    0 ,                       & -\theta_{1} \leq z < \theta_{1}, \\
    (\frac{2\theta_{2} - \theta_{1}}{\theta_{2}})(z + \theta_{1}),            & -\theta_{2} \leq z < -\theta_{1},\\
    z, & z < -\theta_{2},\\
    \end{matrix}\right.
\end{align}
where $\theta_{2} \geq \theta_{1} > 0$.
As a matter of fact, it also can be implemented by the combination of ReLU functions.
Consequently, we define a new ReLU-based activation function as
\begin{gather}\label{MSReLU}
    \begin{split}
    \xi_{(\mathbf{\theta}_{1}, \mathbf{\theta}_{2})}& \left(\mathbf{Z}^{(l)} \right)  \\
    = &w_1 \left( \text{ReLU} \left(\mathbf{Z}^{(l)} - \theta_{1} \right) - \text{ReLU} \left(-\mathbf{Z}^{(l)} - \theta_{1} \right) \right) \\
    - & w_2 \left( \text{ReLU} \left(\mathbf{Z}^{(l)} - \theta_{2} \right) - \text{ReLU} \left(-\mathbf{Z}^{(l)} - \theta_{2} \right) \right),
    \end{split}
\end{gather}
where $w_1$ and $w_2$ are computed according to the parameter settings of $\theta_{1}$ and $\theta_{2}$, that is, 
\begin{gather}
    \begin{split}
    &w_1 = \frac{2\theta_{2} - \theta_{1}}{\theta_{2}}, \\
    &w_2 = w_1 - 1 = \frac{\theta_{2} - \theta_{1}}{\theta_{2}}.
    \end{split}
\end{gather}
Consequently, we have $2 \geq w_1 \geq 1 \geq w_2 \geq 0$.
Eq. \eqref{MSReLU} is termed as a Multi-Stage ReLU (MSReLU) function.
The comparison of MSReLU and other activation functions is shown in Figure \ref{MSReLuCompare}.
It can be observed that with suitable $\theta_{1}$ and $\theta_{2}$, it has less gap between $\xi _{\theta_1, \theta_2} (z)$ obtained by MSReLU and $|z|$ due to the increasing slope when $\theta_{1} < z < \theta_{2}$ and $-\theta_{2} < z < -\theta_{1}$, which is beneficial to obtaining more accurate features.
When $z > \theta_{2}$ and $z < -\theta_{2}$, the slope is the same as ReLU and soft thresholding to maintain the feature distribution of outputs.

\begin{figure}[!tbp]
    \centering
    \includegraphics[width=0.48\textwidth]{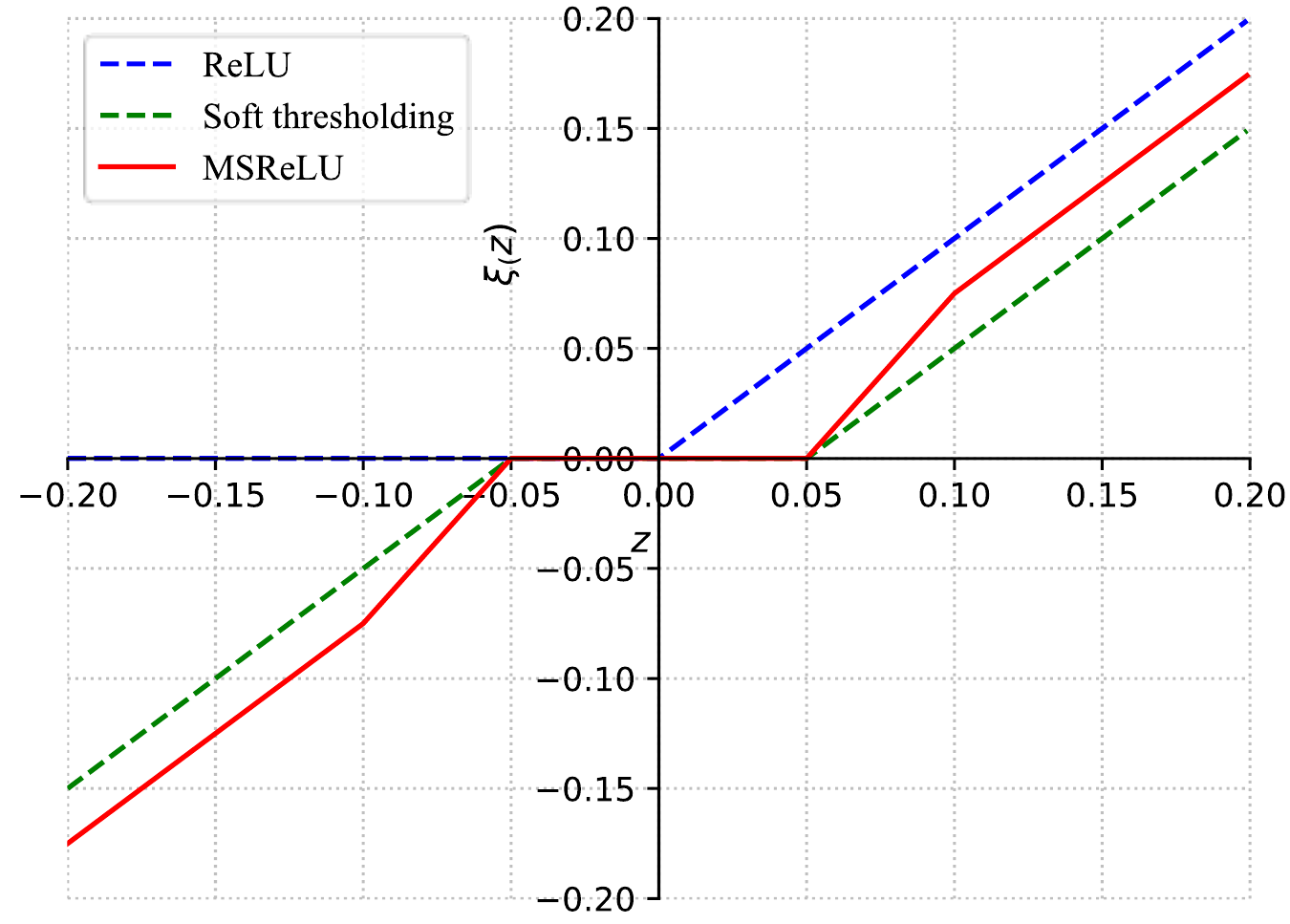}
    \caption{Comparison of different activation functions (ReLU, soft thresholding and MSReLU) for sparse proximal projection, where hyperparameters are fixed as $\theta = 0.05$ for soft thresholding, and $\theta_{1} = 0.05$, $\theta_{2} = 0.10$ for MSReLU.}
    \label{MSReLuCompare}
\end{figure}
Associated with GCL, we can formulate a basic block of the alternating forward computation (contains 2 layers) as
\begin{gather}
    \mathbf{H}^{(l)} = \sigma \left( \tilde{\mathbf{D}}^{-\frac{1}{2}} \tilde{\mathbf{A}} \tilde{\mathbf{D}}^{-\frac{1}{2}} \mathbf{H}^{(l-1)} \mathbf{W}_{g}^{(l)} \right), \label{GCL}\\
    \mathbf{Z}^{(l)} = \xi_{(\mathbf{\theta}_{1}, \mathbf{\theta}_{2})} \left(\mathbf{H}^{(l)} \mathbf{W}_{e1}^{(l)} + \mathbf{X} \mathbf{W}_{e2}^{(l)} - \lambda \tilde{\mathbf{L}} \mathbf{H}^{(l)}  \right), \label{GEL}
\end{gather}
where $\mathbf{H}^{(l-1)} = \mathbf{Z}^{(l-1)}$ for $l = 2, \ldots, t$ and $\mathbf{H}^{(0)} = \mathbf{X}$.
We term the forward computation defined in Eq. \eqref{GEL} as Graph Embedding Layer (GEL).
The definition of GEL shows that it refines graph representations from the previous GCL and considers one-hop embeddings of neighbors via $\tilde{\mathbf{L}}\mathbf{H}^{(l)}$.
Here we adopt $\mathbf{H}^{(l)}$ generated by GCL as the input of GEL, because GCL also implicitly optimizes the graph Laplacian regularization term.
Actually, both GCL and GEL are one-step approximations of Laplacian-based graph regularization problems.
GEL also leverages the information of original features via an input-injected computation defined by $\mathbf{X} \mathbf{W}_{e2}^{(l)}$ to preserve sparse and discriminative representations of nodes at the hidden and the last layers, thereby alleviating the over-smoothing problem.
On the basis of Eqs. \eqref{GCL} and \eqref{GEL}, we can construct a deep block-wise graph neural network with $2t$ layers that consists of GCL and GEL alternately.

\subsection{Alternating Graph-regularized Neural Network with Improved Adaboost}
\label{AGNN2}
In order to further leverage underlying features at each layer and obtain results contributed by different hops of neighborhood relationships, we adopt a variant of Adaboost to compute the final predictions of the model.
For the purpose of obtaining graph representations with the same dimension, we adopt a weak classifier
\begin{equation}
    c \left(\mathbf{H}^{(l)} \right) = \text{Softmax} \left( \sigma \left(\mathbf{H}^{(l)}\mathbf{W}_{c} + b \right) \right)
\end{equation}
for each layer of GCL, where $\mathbf{W}_{c} \in \mathbb{R}^{d_{l} \times d_{L}}$.
The weak classifier $c \left(\mathbf{Z}^{(l)} \right)$ for GEL is homologous.
We assign corresponding weights $\alpha^{(l)}$ and $\beta^{(l)}$ for each GCL and GEL.
Formulaically, the final weighted result of various classifiers is
\begin{gather}\label{WeightedResult}
    \mathbf{S} = \sum_{l=1}^{t} \left( \alpha^{(l)} c \left(\mathbf{H}^{(l)} \right) + \beta^{(l)} c \left(\mathbf{Z}^{(l)} \right) \right),
\end{gather}
where $\alpha^{(l)}$ indicates the weight of classifier w.r.t. $\mathbf{H}^{(l)}$ and $\beta^{(l)}$ indicates the weight of classifier w.r.t. $\mathbf{Z}^{(l)}$.
We measure the performance of each weak classifier on labeled nodes to calculate classifier weights, which ensures that classifiers with higher accuracy on the training set are assigned to larger weights.
First, the weighted error rates of two types of classifiers are computed by
\begin{gather}
    e_{\mathbf{H}}^{(l)} = \sum_{i \in \Omega} \pi_{i} \mathbb{I} \left( c \left(\mathbf{H}^{(l)} \right)_{i} \neq y_{i} \right) / \sum_{i \in \Omega} \pi_{i}, \label{WeightedErrGCL} \\
    e_{\mathbf{Z}}^{(l)} = \sum_{i \in \Omega} \pi_{i} \mathbb{I} \left( c \left(\mathbf{Z}^{(l)} \right)_{i} \neq y_{i} \right) / \sum_{i \in \Omega} \pi_{i}, \label{WeightedErrGEL}
\end{gather}
where $\Omega$ is the set of samples having supervision information and $\pi_{i}$ is the weight of a labeled node.
The sample weights are initialized by $\pi_{i} = \frac{1}{|\Omega|}$.
Therefore, classifier weights $\alpha^{(l)}$ and $\beta^{(l)}$ are computed by
\begin{gather}
    \alpha^{(l)} = \frac{1}{2} log \frac{1-e_{\mathbf{H}}^{(l)}}{e_{\mathbf{H}}^{(l)}} + log(R - 1), \label{ClassifierWeightGCL}\\
    \beta^{(l)} = \frac{1}{2} log \frac{1-e_{\mathbf{Z}}^{(l)}}{e_{\mathbf{Z}}^{(l)}} + log(R - 1), \label{ClassifierWeightGEL}
\end{gather}
where $R$ is the number of classes.
We apply the softmax normalization to all classifier weights, i.e.,
\begin{gather}\label{softmaxnormalize}
    [ \bm{\alpha}, \bm{\beta} ] \leftarrow \text{Softmax} ([ \bm{\alpha}, \bm{\beta} ]),
\end{gather}
where $\bm{\alpha} = [\alpha^{(1)}, \cdots, \alpha^{(l)}]$ and $\bm{\beta} = [\beta^{(1)}, \cdots, \beta^{(l)}]$.
For the purpose of increasing weights on incorrect classified nodes, we update $\pi_{i}$ by
\begin{gather}
        \pi_{i} \leftarrow (1 + \eta_{i}) \pi_{i}  \mathbb{I} \left( c_{i} \neq y_{i} \right),  \label{IncorrectUpdate} \\
        \pi_{i} \leftarrow \max (1 - \eta_{i}, \rho) \pi_{i} \mathbb{I} \left( c_{i} = y_{i} \right), \label{CorrectUpdate}
\end{gather} 
where $c_{i}$ is the predicting result of the former classifier and $y_{i}$ is the ground truth.
$\eta_{i}$ is an updating rate that changes the sample weight automatically according to predictions of the weak classifier.
The threshold $0 < \rho < 1$ is adopted to avoid nodes with weights of zeros.
In particular, the updating rate $\eta_{i}$ applied in this paper is defined by
\begin{gather}\label{SampleweightsUpdate}
    \eta_{i} = exp \left( log \left( \frac{p_{i, r}}{ \max \left(\sum_{j=1, j \neq r}^{R} p_{i,j}, \epsilon \right)} \right) \right),
\end{gather}
where $p_{i, r}$ is the probability of the $i$-th sample belonging to the $r$-th class and is obtained from the $r$-th entry of $\left[c \left( \mathbf{H}^{(l)} \right) \right]_{i}$ or $c \left[ \left( \mathbf{Z^{(l)}} \right)\right]_{i}$.
Namely, $p_{i, r} = \left[c\left( \mathbf{H}^{(l)} \right) \right]_{i, r}$ or $p_{i, r} = \left[c\left( \mathbf{Z}^{(l)} \right) \right]_{i, r}$.
Here $\epsilon$ is a tiny value avoiding the divide-by-zero error.
A higher $\eta_{i}$ indicates that the importance of the $i$-th sample should be larger if it is incorrectly classified, and should be smaller otherwise.
For a correctly predicted node, the weight of it would decrease remarkably if $p_{i, r}$ is higher.
This indicates that the model should pay less attention to correct predictions with high confidence.
As for a misclassified node, the weight of it would grow up considerably with higher $p_{i, r}$, attributed to the reason that the node prediction result is much against the ground truth.

With the weighted node embedding obtained by Eq. \eqref{WeightedResult}, the objective of the proposed AGNN is the cross-entropy loss function, i.e.,
\begin{gather}\label{GCNLoss}
    \mathcal{L} = - \sum_{i \in \Omega} \sum_{j=1}^{c} \mathbf{Y}_{ij} \mathrm{ln} \mathbf{S}_{ij},
\end{gather}
which only works on nodes in the training set $\Omega$ to perform the semi-supervised classification task.

\subsection{Model Analysis}
\label{AGNN3}

Algorithm \ref{AGCNalgo} depicts the procedure of AGNN.
In general, the procedure of AGNN is divided into two parts:
forward computation of multiple network layers and calculation on weighted graph embedding $\mathbf{S}$ via the variant of Adaboost. 
Given weight matrix $\mathbf{W}_{g}^{(l)} \in \mathbb{R}^{d_{l-1} \times d_{l}}$,
the computational complexity for the $l$-th GCL is linear to the number of edges $|\mathcal{E} |$.
Namely, it is $\mathcal{O} (|\mathcal{E} | d_{l-1}d_{l})$.
As to the $l$-th GEL, the computational complexity is $\mathcal{O} ( |\mathcal{E} |d_{l} + nmd_{l})$.
Consequently, the forward computation of a basic block with a GCL and a GEL is approximately $\mathcal{O} ( |\mathcal{E} | d_{l-1}d_{l} + nmd_{l})$.
Owing to $d_{l} \ll \min (n, m)$, GEL does not significantly increase the computational cost of the networks.

\begin{algorithm}[tbp]
    \renewcommand{\algorithmicrequire}{\textbf{Input:}}
    \renewcommand{\algorithmicensure}{\textbf{Output:}}
    \caption{Alternating Graph-regularized Neural Network}
    \begin{algorithmic}\label{AGCNalgo}
    \REQUIRE {Adjacency matrix $\mathbf{A} \in \mathbb{R}^{n \times n}$, feature matrix $\mathbf{X} \in \mathbb{R}^{n \times m}$, hyperparameters $\lambda$, $\rho$, $\theta_1$ and $\theta_{2}$;}
    \ENSURE Graph embedding $\mathbf{S} \in \mathbb{R}^{n \times c }$.
        \WHILE {not convergent}
            \FOR {$l = 1 \rightarrow t$}
                \STATE {Compute the output $\mathbf{H}^{(l)}$ of the $l$-th GCL via Eq. \eqref{GCL};}
                \STATE {Compute the output $\mathbf{Z}^{(l)}$ of the $l$-th GEL via Eq. \eqref{GEL};}
            \ENDFOR 
            \STATE {Initialize weights $\{ \pi_{i} \}_{i \in \Omega}$ by $\pi_{i} = \frac{1}{|\Omega|}$;}
            \FOR {$l = 1 \rightarrow t$}
                \STATE {Update sample weights $\{ \pi_{i} \}_{i \in \Omega}$ for the $l$-th GCL via Eqs. \eqref{IncorrectUpdate}, \eqref{CorrectUpdate} and \eqref{SampleweightsUpdate};}
                \STATE {Calculate the classifier weight $\alpha^{(l)}$ for the $l$-th GCL via Eqs. \eqref{WeightedErrGCL} and \eqref{ClassifierWeightGCL};}
                \STATE {Update sample weights $\{ \pi_{i} \}_{i \in \Omega}$ for the $l$-th GEL via Eqs. \eqref{IncorrectUpdate}, \eqref{CorrectUpdate} and \eqref{SampleweightsUpdate};}
                \STATE {Calculate the classifier weight $\beta^{(l)}$ for the $l$-th GEL via Eqs. \eqref{WeightedErrGEL} and \eqref{ClassifierWeightGEL};}
            \ENDFOR
            \STATE {Obtain weighted embeddings $\mathbf{S}$ via Eqs. \eqref{softmaxnormalize} and \eqref{WeightedResult};}
            \STATE {Update all trainable parameters via back propagation;}
        \ENDWHILE
    \RETURN {Weighted graph embedding $\mathbf{S}$.}
    \end{algorithmic}
    \end{algorithm}

In light of previous analysis, both GCL and GEL are approximations of optimization problems w.r.t. graph regularization, attributed to which they can be considered as two distinct layers.
Hence, AGNN can be approximately regarded as an alternating optimization procedure of Problems \eqref{OptimizationGCL} and \eqref{GraphEmbedding}.
The difference between the two layers is that the former optimization performs graph convolutions, and the latter optimization is a sparse graph-regularized projection from the original feature space.
In a nutshell, the proposed AGNN is a block-wise graph neural network that simultaneously considers cross-layer connection and aggregation of multi-hop information, which is beneficial to obtaining reliable high-order neighborhood embeddings before conducting information fusion.
The primary differences to existing models are summarized as follows:
\begin{enumerate}
    \item Different from methods that directly combine node embeddings from outputs of varied layers (e.g., AdaGCN \cite{SunZL21}), AGNN gets rid of inaccurate predictions of deep layers via periodic projection from original feature space to latent embeddings.
    \item Instead of widely used additive connections from previous layers, AGNN establishes an optimization-inspired GEL module which is derived from the Laplacian-based graph regularization problem. This makes the initial features propagate to each GCL dexterously with less information loss.
\end{enumerate}

\section{Experimental Analysis}\label{Experiments}
In this subsection, comprehensive experiments are conducted including evaluation against several state-of-the-art models and ablation studies.
All experiments are run on a platform with AMD R9-5900X CPU, NVIDIA GeForce RTX 3060 12G GPU and 32G RAM.

\subsection{Experimental Setup}
For the following experiments, we compare the proposed AGNN with numerous methods.
Apart from classical baselines (MLP and Chebyshev \cite{DefferrardBV16}), other state-of-the-art methods can be divided into two categories: vanilla GNN-based models (GraphSAGE \cite{HamiltonYL17}, GAT \cite{VelickovicCCRLB18} and ScatteringGCN \cite{MinWW20}), and multi-layer or high-order-information-based GCN methods (GCN \cite{KipfW17}, APPNP \cite{KlicperaBG19}, JK-Net \cite{XuLTSKJ18}, SGC \cite{WuSZFYW19}, ClusterGCN \cite{ChiangLSLBH19}, GCNII \cite{ChenWHDL20}, SSGC \cite{ZhuK21} and AdaGCN \cite{SunZL21}).
In particular, APPNP, SGC and SSGC propagate node information via the proposed high-order filters, where numbers of order can be regarded as numbers of layers for other multi-layer approaches.
The compared models are demonstrated in detail as follows.

\begin{table*}[!htbp]
    \centering
    \begin{tabular}{l|cccccccc}
    \toprule
     Datasets       & \# Nodes & \# Edges  & \# Features   & \# Classes & \# Train  &   \# Valid  & \# Test & Data types\\
     \midrule
     Citeseer       & 3,327    & 4,732   & 3,703    & 6   & 120   & 500 & 1,000 & Citation network \\
     CoraFull       & 19,793   & 63,421  & 8,710    & 70  & 1,400  & 500 & 1,000 & Citation network      \\
     Chameleon      & 2,277    & 18,050  & 2,325    & 5   & 100   & 500 &  1,000 & Link network \\
     BlogCatalog    & 5,196    & 171,743 & 8,189    & 6   & 120    & 500 &  1,000 & Social network       \\
     ACM            & 3,025    & 13,128  & 1,870    & 3   & 60     & 500 &  1,000 & Paper network       \\ 
     Flickr         & 7,575    & 239,738 & 12,047   & 9   & 180   & 500 & 1,000 & Social network     \\
     UAI            & 3,067    & 28,311  & 4,973    & 19  & 380  & 500 &  1,000 & Citation network      \\
     Actor           & 7,600   & 15,009   & 932     & 5  & 100 & 500 &  1,000 & Social network     \\
    \bottomrule
    \end{tabular}
    \caption{A brief statistics of all graph datasets and data split modes.}
    \label{DataDescription}
\end{table*}

\begin{enumerate}
    \item \textbf{MLP} is a classical baseline for classification, which is a multi-layer perceptron architecture with a softmax function as the classifier.
    \item \textbf{Chebyshev} is a GCN-like baseline that adopts Chebyshev filters to perform graph convolutions with the given node features and the topology network.
    \item \textbf{GCN} conducts a variant of convolution on the graph, which is exactly the first-order approximation of the Chebyshev polynomial.
    \item \textbf{GraphSAGE} constructs a graph neural network that explores node embeddings through sampling and accumulating features from local neighbors of a node.
    \item \textbf{GAT} is a graph neural network adopting an attention mechanism to explore node attributes across the graph, which enables the implicit assignment of weights to distinct nodes in a neighborhood.
    \item \textbf{JK-Net} dexterously exploits various neighborhood ranges of nodes via a jumping knowledge structure that considers residual connections.
    \item \textbf{SGC} proposes a faster variant of GCN via successively removing nonlinearities and collapsing weight matrices between consecutive layers.
    \item \textbf{APPNP} leverages personalized PageRank to improve the performance of GCN-like models, which derives an improved propagation scheme.
    \item \textbf{ClusterGCN} is a GCN-based framework that samples a group of nodes by a graph clustering algorithm, which alleviates the over-smoothing problem via a diagonal enhancement architecture.
    \item \textbf{GCNII} is a variant of GCN with residual connection and identity mapping, which effectively alleviates the over-smoothing phenomenon.
    \item \textbf{ScatteringGCN} builds an augmented GCN with geometric scattering transforms and residual convolutions to alleviate the over-smoothing issue.
    \item \textbf{SSGC} develops a variant of GCN by adopting a modified Markov diffusion kernel, which explores the global and local contexts of nodes.
    \item \textbf{AdaGCN} integrates learned knowledge from distinct layers of GCN in an Adaboost way, which updates layer weights iteratively.
\end{enumerate}

\begin{table*}[!tbp]
    \centering
    \begin{tabular}{l|llllllll|l}
        \toprule
        Methods / Datasets & Citeseer & CoraFull & Chameleon & BlogCatalog & ACM    & Flickr & UAI   & Actor  & Avg Ranks \\ \midrule
        MLP               &  0.366  & 0.051    & 0.286     & 0.646       & 0.812  & 0.431  & 0.188  & 0.224 &  13.3 \\
        Chebyshev \cite{DefferrardBV16} &  0.693  & 0.534    & 0.217     & 0.357       & 0.829  & 0.304  & 0.215  & 0.182 & 13.4  \\\midrule
        GraphSAGE \cite{HamiltonYL17}  &  0.620   & 0.521    & 0.437     & 0.525       &  0.886  & 0.286  & 0.483  & 0.191 & 13.0  \\
        GAT \cite{VelickovicCCRLB18}     &  0.683  & 0.571    & 0.460     & 0.681       & 0.889  & 0.429  & 0.597  & 0.246 &  7.38 \\
        ScatteringGCN \cite{MinWW20}   & 0.679    & 0.519    & 0.410     & 0.690       & 0.890  & 0.419  & 0.364  & 0.214 & 11.4  \\
        GCN \cite{KipfW17}    & 0.697 (2)  & 0.567 (2)   & 0.447 (2)    & 0.711  (2)      & 0.875 (2) & 0.414 (2)  & 0.498 (2)  & 0.240 (4)  & 9.25 \\
        JK-Net \cite{XuLTSKJ18} & 0.703 (4)    & 0.568 (2)    & 0.475 (20)    & 0.747 (16)    & 0.892 (8)  & 0.547 (2)  & 0.494 (18)  & 0.224 (18) & 6.38  \\
        APPNP \cite{KlicperaBG19} & 0.698 (4)  & 0.576 (4)  & 0.404 (2) & 0.813 (8) & 0.885 (4)  & 0.521 (2)  & 0.560 (2)  & 0.212 (4)  &  7.75\\
        SGC \cite{WuSZFYW19} & 0.697 (10)  & 0.583 (2)  & 0.445 (2)   & 0.716 (2)     & 0.887 (2) & 0.410 (2)  & 0.571 (2)  & 0.247 (4) & 7.50  \\
        ClusterGCN \cite{ChiangLSLBH19} & 0.681 (2) & 0.576 (2) &  0.449 (2) & 0.731 (2)  & 0.893 (2) & 0.483 (2)  & 0.525 (2)  & 0.239 (8) & 6.75 \\
        GCNII \cite{ChenWHDL20}   & \textcolor{red}{\textbf{0.710 (12)}} &  0.576 (18) &   0.449 (16) &   0.845 (12)  &  0.901 (12) &  0.545 (12)  &  0.619 (14)  &  0.238 (10) &  3.50 \\
        SSGC \cite{ZhuK21}    & 0.702 (20)  & 0.575 (4) & 0.446 (2)    & 0.760 (2)     & 0.889 (2)  & 0.478 (2)  & 0.523 (10)  & 0.248 (2) &  6.50 \\
        AdaGCN \cite{SunZL21}& 0.663 (2)  & 0.587 (10) & 0.479 (4)    & 0.800 (2)      & 0.894 (2) & 0.552 (2)  & 0.588 (2)  & 0.230 (2) & 4.88  \\\midrule
        AGNN w/o Adaboost & 0.689 (2)   & 0.564 (2) & 0.440 (2) & 0.766 (10)    & 0.888 (2) & 0.503 (20)  & 0.574 (10) & 0.254 (16) &  7.13 \\
        AGNN              & 0.707 (6) & \textcolor{red}{\textbf{0.589 (6)}}  & \textcolor{red}{\textbf{0.503 (14)}}   & \textcolor{red}{\textbf{0.849 (10)}}   &  \textcolor{red}{\textbf{0.903 (8)}} & \textcolor{red}{\textbf{0.584 (4)}} & \textcolor{red}{\textbf{0.647 (6)}} & \textcolor{red}{\textbf{0.256 (6)}} & \textcolor{red}{\textbf{1.13}}  \\ 
        \bottomrule
    \end{tabular}
    \caption{Performance (accuracy) comparison with 20 labeled samples per class as supervision signals, where the highest accuracy is highlighted in red. The last column shows the average ranks of the performance of different methods. For multi-layer or multi-order information-based models, the optimal layer numbers or orders are recorded in brackets.}
    \label{Performance}
\end{table*}
In this paper, eight different graph-structural datasets are adopted to evaluate the performance of numerous methods, as listed below:
\begin{enumerate}
    \item \textbf{Citeseer}\footnote{https://linqs.soe.ucsc.edu/data} is a benchmark dataset for literature citation networks, where nodes represent papers and edges represent citations between them.
    \item \textbf{CoraFull}\footnote{https://github.com/shchur/gnn-benchmark\#datasets} is the larger version of Cora dataset, which is another well-known citation network. Herein, each node denotes paper and edge stands for citation. All nodes are classified according to their topics.
    \item \textbf{Chameleon}\footnote{https://github.com/benedekrozemberczki/MUSAE/} contains node relationships of a large number of articles on a topic of the English Wikipedia website, where edges represent the mutual links among articles.
    \item \textbf{BlogCatalog}\footnote{https://networkrepository.com/soc-BlogCatalog.php} includes a large number of bloggers and their social relationships from the website. Node features are extracted from the keywords of user information and all bloggers are divided into 6 distinct types.
    \item \textbf{ACM}\footnote{https://github.com/Jhy1993/HAN} is a paper network where each node denotes a paper. Different from citation networks, edges connect papers that share the same authors.
    \item \textbf{Flickr}\footnote{https://github.com/xhuang31/LANE} is a social network that records relationships among users from an image and video hosting website. All users are grouped into 9 categories on the basis of their personal interests.
    \item \textbf{UAI}\footnote{https://github.com/zhumeiqiBUPT/AM-GCN} is a dataset for the test of GCN on community detection, which is a webpage citation network. Nodes representing webpages are collected from multiple universities and each edge denotes the citation.
    \item \textbf{Actor}\footnote{https://github.com/CUAI/Non-Homophily-Large-Scale} is a subgraph of the film-director-actor-writer network, which only includes the connections of various actors. Each edge represents the co-occurrence of two actors on the same Wikipedia page.
\end{enumerate}

A statistical summary of these datasets is demonstrated in Table \ref{DataDescription}.
For fair comparison and avoiding undesired influence raised by data distribution, we shuffle all datasets and randomly select 20 labeled samples per class for training, 500 samples for validation and 1,000 samples for testing.

In order to provide a fair test bed for all compared methods, we list some hyperparameters in experiments.
Learning rates of these methods are fixed as $0.01$ or $0.005$, which are preferred to be smaller when more network layers are utilized.
For all GNN-based methods, we fix the number of hidden units at each layer as 128 or 16.
Other method-specific hyperparameters are fixed as their settings in original papers.

As for the proposed AGNN, we also apply the same hidden layers as compared methods.
The learning rates are also selected from 0.01 and 0.005.
In general, a deeper AGNN requires a smaller learning rate, and we adopt learning rate adaptation via decreasing it when there is no loss drop for a period of training epochs.
The Adam optimizer is adopted and the weight decay is fixed as $5 \times 10^{-4}$.
The activation function $\sigma(\cdot)$ is $\text{tanh}(\cdot)$ for weak classifiers while $\text{ReLU}(\cdot)$ for GCL.
As for the thresholds in the MSReLU function of GEL, we fix them as $\theta_1 = 0.02$ and $\theta_2 = 0.04$.
For the Adaboost strategy, the tiny value in Eq. \eqref{SampleweightsUpdate} is fixed as $10^{-4}$.

\subsection{Experimental Results}

\subsubsection{Performance comparison}
First of all, we compare the performance of the proposed AGNN with all selected approaches.
Table \ref{Performance} exhibits the semi-supervised classification accuracy on eight datasets.
In pursuit of conducting the ablation study and validating the effectiveness of the designed network structure,
we further examine the performance of AGNN without the Adaboost framework (dubbed AGNN w/o AdaBoost), which does not aggregate embeddings of all network layers but directly outputs predictions of the final GEL.
Because multi-layer or multi-order information-based models aim to improve GCN via mining information from deep layers, we record the highest accuracy of these models and the corresponding numbers of layers.
The optimal numbers of layers or orders of neighbors are shown in brackets.

In order to validate the statistical significance of the experimental results, we follow \cite{ShangYLZ22} and adopt Friedman test.
The average ranks of all compared models are recorded in the last column of Table \ref{Performance}, on the basis of which we obtain the Friedman testing score $F_{F} = 10.57$.
With $15$ compared models and $8$ test datasets, the critical value is $1.794$ for $\alpha = 0.05$,
which indicates that $F_{F}$ is higher than the critical value.
Thus, we can reject the null hypothesis, which points out that the performance of all compared methods is significantly different with a confidence level at $95\%$.

\begin{table*}[!tbp]
    \centering
    \begin{tabular}{c|l|llllllllll}
        \toprule 
        Datasets & Models & 2-layer     & 4-layer     & 6-layer & 8-layer & 10-layer & 12-layer & 14-layer & 16-layer & 18-layer & 20-layer \\ \midrule
        \multirow{9}{*}{CoraFull}
        & GCN \cite{KipfW17}              & \textbf{0.567}* & 0.495 & 0.451 & 0.443 & 0.408 & 0.376 & 0.332 & 0.204 & 0.119 & 0.019 \\ 
        & JK-Net \cite{XuLTSKJ18} & \textbf{0.568}*             & 0.534 & 0.531 & 0.493 & 0.553 & 0.506 & 0.456 & 0.523 & 0.527 & 0.530\\ 
        & APPNP \cite{KlicperaBG19}            & 0.569 & \textcolor{red}{\textbf{0.576}*} & 0.560 & 0.561 & 0.550 & 0.556 & 0.552 & 0.549 & 0.547 & 0.544 \\ 
        & SGC \cite{WuSZFYW19}            & \textcolor{red}{\textbf{0.583}*} & \textcolor{red}{\textbf{0.576}} & 0.562 & 0.551 & 0.534 & 0.512 & 0.495 & 0.474 & 0.441 & 0.418 \\ 
        & ClusterGCN \cite{ChiangLSLBH19} & \textbf{0.576}* & 0.518 & 0.494 & 0.475 & 0.442 & 0.391 & 0.337 & 0.264 & 0.257 & 0.214 \\ 
        & GCNII \cite{ChenWHDL20} & 0.539 & 0.536 & 0.558 & 0.565 & 0.568 & 0.571 & 0.568 & 0.574 & \textbf{0.576}* & 0.565 \\ 
        & SSGC  \cite{ZhuK21}             & 0.572 & \textbf{0.575}* & 0.572 & 0.561 & 0.562 & 0.561 & 0.541 & 0.564 & 0.562 & 0.537 \\ 
        & AdaGCN \cite{SunZL21} & 0.552             & 0.553 & 0.571 & 0.573 & \textcolor{red}{\textbf{0.587}*} & \textcolor{red}{\textbf{0.579}} & \textcolor{red}{\textbf{0.586}} & 0.575 & 0.564 & 0.535  \\ 
        & AGNN w/o Adaboost    & \textbf{0.564}*     & 0.544 & 0.532 & 0.554 & 0.544 & 0.554 & 0.523 & 0.536 & 0.545 & 0.541   \\ 
        & AGNN   & 0.570 & 0.574 & \textcolor{red}{\textbf{0.589}*} & \textcolor{red}{\textbf{0.583}} & 0.580 & 0.568 & 0.565 & \textcolor{red}{\textbf{0.577}} & \textcolor{red}{\textbf{0.584}} & \textcolor{red}{\textbf{0.574}} \\ \toprule
        \multirow{9}{*}{BlogCatalog}
        & GCN \cite{KipfW17}              & \textbf{0.697}*   & 0.548 & 0.231 & 0.125 & 0.142 & 0.154 & 0.187 & 0.164 & 0.159 & 0.171  \\ 
        & JK-Net \cite{XuLTSKJ18}         & 0.725             & 0.711 & 0.693 & 0.711 & 0.670 & 0.724 & 0.696 & \textbf{0.747}* & 0.668 & 0.698 \\ 
        & APPNP \cite{KlicperaBG19}   & 0.791 & 0.810 & 0.811 & \textcolor{red}{\textbf{0.813}*} & 0.809 & 0.806 & 0.809 & 0.811 & 0.805 & 0.804 \\ 
        & SGC \cite{WuSZFYW19}            & \textbf{0.716}*           & 0.616 & 0.490 & 0.394 & 0.313 & 0.238 & 0.232 & 0.225 & 0.220 & 0.237 \\ 
        & ClusterGCN \cite{ChiangLSLBH19} & \textbf{0.731}*             & 0.542 & 0.395 & 0.256 & 0.171 & 0.192 & 0.182 & 0.176 & 0.172 & 0.171 \\ 
        & GCNII \cite{ChenWHDL20} & \textcolor{red}{\textbf{0.816}} & 0.813 & 0.799 & 0.802 & 0.843 & \textcolor{red}{\textbf{0.845}*} & 0.810 & \textcolor{red}{\textbf{0.838}} & 0.801 & 0.796 \\ 
        & SSGC  \cite{ZhuK21}             & \textbf{0.760}*  & 0.744 & 0.744 & 0.736 & 0.683 & 0.728 & 0.726 & 0.723 & 0.722 & 0.661 \\ 
        & AdaGCN \cite{SunZL21}        & \textbf{0.800}*   & 0.723 & 0.678 & 0.682 & 0.681 & 0.684 & 0.678 & 0.688 & 0.684 & 0.688   \\ 
        & AGNN w/o Adaboost   & 0.762             & 0.745 & 0.746 & 0.741 & \textbf{0.766}* & 0.736 & 0.737 & 0.751 & 0.754 & 0.748 \\
        & AGNN                             & 0.720             & \textcolor{red}{\textbf{0.824}} & \textcolor{red}{\textbf{0.824}} & 0.805 & \textcolor{red}{\textbf{0.849}*} & \textbf{0.815} & \textcolor{red}{\textbf{0.820}} & 0.814 & \textcolor{red}{\textbf{0.808}} & \textcolor{red}{\textbf{0.814}} \\ \toprule
        \multirow{9}{*}{Flickr}
        & GCN \cite{KipfW17}              & \textbf{0.414}* & 0.127 & 0.161 & 0.091 & 0.100 & 0.092 & 0.089 & 0.091 & 0.094 & 0.095 \\ 
        & JK-Net \cite{XuLTSKJ18}    & \textbf{0.547}*             & 0.421 & 0.392 & 0.418 & 0.422 & 0.409 & 0.439 & 0.445 & 0.343 & 0.345\\ 
        & APPNP \cite{KlicperaBG19}   & \textbf{0.521}* & 0.502 & 0.461 & 0.485 & 0.465 & 0.475 & 0.468 & 0.474 & 0.487 & 0.464 \\ 
        & SGC \cite{WuSZFYW19}            & \textbf{0.410}* & 0.337 & 0.220 & 0.197 & 0.154 & 0.179 & 0.167 & 0.160 & 0.155 & 0.154 \\ 
        & ClusterGCN \cite{ChiangLSLBH19} & \textbf{0.483}* & 0.398 & 0.314 & 0.322 & 0.217 & 0.198 & 0.184 & 0.143 & 0.112 & 0.103 \\ 
        & GCNII \cite{ChenWHDL20} & 0.489 & 0.499 & 0.514 & 0.511 & 0.530 & \textcolor{red}{\textbf{0.545}*} & 0.523 & 0.538 & 0.524 & 0.524 \\ 
        & SSGC  \cite{ZhuK21}             & \textbf{0.478}* & 0.433 & 0.388 & 0.356 & 0.340 & 0.328 & 0.320 & 0.315 & 0.309 & 0.304 \\ 
        & AdaGCN \cite{SunZL21}   & \textbf{0.552}*             & 0.546 & \textcolor{red}{\textbf{0.539}} & \textcolor{red}{\textbf{0.539}} & 0.542 & \textcolor{red}{\textbf{0.545}} & \textcolor{red}{\textbf{0.545}} & \textcolor{red}{\textbf{0.545}} & 0.544 & 0.546 \\ 
        & AGNN w/o Adaboost   & 0.481             & 0.494 & 0.488 & 0.490 & 0.495 & 0.494 & 0.499 & 0.495 & 0.493 & \textbf{0.503}* \\ 
        & AGNN       &      \textbf{\textcolor{red}{0.560}} & \textcolor{red}{\textbf{0.584}*} & 0.529 & 0.521 & \textcolor{red}{\textbf{0.543}} & 0.511 & 0.522 & 0.535 & \textcolor{red}{\textbf{0.545}} & \textcolor{red}{\textbf{0.557}} \\ \toprule
        \multirow{9}{*}{UAI}
        & GCN \cite{KipfW17}              & \textbf{0.498}* & 0.301 & 0.195 & 0.202 & 0.175 & 0.186 & 0.192 & 0.123 & 0.109 & 0.080 \\ 
        & JK-Net \cite{XuLTSKJ18}  & 0.474             & 0.467 & 0.466 & 0.492 & 0.467 & 0.485 & 0.490 & 0.484 & \textbf{0.494}* & 0.476\\ 
        & APPNP \cite{KlicperaBG19}  & \textbf{0.560}* & 0.507 & 0.484 & 0.531 & 0.520 & 0.510 & 0.493 & 0.540 & 0.526 & 0.486 \\ 
        & SGC \cite{WuSZFYW19}            & \textbf{0.571}* & 0.481 & 0.258 & 0.141 & 0.136 & 0.157 & 0.126 & 0.080 & 0.072 & 0.036 \\ 
        & ClusterGCN \cite{ChiangLSLBH19} & \textbf{0.522}* & 0.516 & 0.452 & 0.358 & 0.244 & 0.265 & 0.224 & 0.198 & 0.167 & 0.165 \\
        & GCNII \cite{ChenWHDL20} & \textcolor{red}{\textbf{0.598}} & 0.602 & 0.611 & 0.608 & 0.617 & \textbf{0.622}* & 0.619 & 0.618 & 0.619 & 0.615 \\ 
        & SSGC  \cite{ZhuK21}             & 0.508 & 0.500 & 0.509 & 0.522 & \textbf{0.523}* & 0.518 & 0.519 & 0.523 & 0.519 & 0.521 \\ 
        & AdaGCN \cite{SunZL21}   & \textbf{0.588}*            & 0.583 & 0.581 & 0.582 & 0.582 & 0.580 & 0.579 & 0.576 & 0.581 & 0.582 \\ 
        & AGNN w/o Adaboost    & 0.567             & 0.561 & 0.551 & 0.565 & \textbf{0.574}* & 0.562 & 0.560 & 0.562 & 0.545 & 0.563   \\ 
        & AGNN                            & 0.572 & \textcolor{red}{\textbf{0.630}} & \textcolor{red}{\textbf{0.647}*} & \textcolor{red}{\textbf{0.640}} & \textcolor{red}{\textbf{0.641}} & \textcolor{red}{\textbf{0.638}} & \textcolor{red}{\textbf{0.623}} & \textcolor{red}{\textbf{0.623}} & \textcolor{red}{\textbf{0.622}} & \textcolor{red}{\textbf{0.621}} \\ \bottomrule
    \end{tabular}
    \caption{Accuracy comparison (GCN, SGC, ClusterGCN, SSGC, AdaGCN and AGNN) with various numbers of layers on Chameleon, CoraFull, Flickr and UAI datasets. The optimal numbers of layers for each method are highlighted with ``*", and the best performance of different methods with the same number of layers is highlighted in red.}
    \label{Deeperlayer}
\end{table*}

From experimental results, we draw the following conclusions.
First, the experimental results indicate that the proposed AGNN attains encouraging performance and outperforms the other methods by a considerable margin on most datasets.
Second, it can be observed that AGNN obtains the optimal classification accuracy with more layers.
In most cases, AGNN achieves the best performance with more than 6 layers.
Although other compared methods sometimes also achieve better performance with more layers, 2 or 4 layers are still the best choice for most datasets.
Last but not the least, AGNN w/o Adaboost obtains competitive classification results and sometimes gets higher accuracy with over 10 layers (Flickr, UAI and Actor datasets).
This phenomenon indicates that AGNN w/o Adaboost guarantees the discrimination of node embeddings and the reliability of deep layers.
From the ablation study, we find that AGNN performs satisfactorily compared with AGNN w/o Adaboost, which indicates the effectiveness of the proposed improved Adaboost.
In addition, the experimental results point out that the optimal layer numbers of AGNN are not always higher than that of AGNN w/o Adaboost.
This may be owing to the fact that the aggregation process enables the model to obtain competitive accuracy with fewer layers.
Besides, deep-layer models do not always mine more information on some datasets, which depends on the topological structure of datasets.
However, it is significant that AGNN obtains higher accuracy with more layers on some datasets, especially on Chameleon and ACM.
In a nutshell, these observations reveal that the performance leading of AGNN is significant with larger numbers of layers.

\subsubsection{Performance with deep layers}
Because the proposed AGNN aims to tackle the over-smoothing issue and extract more distinctive characteristics with deep layers, we further conduct comparing experiments on some multi-layer or multi-order information-based GCN methods to explore accuracy trends with varying numbers of layers.

Table \ref{Deeperlayer} demonstrates the classification accuracy of selected methods with 20 labeled nodes for each class, from which we have the following observation.
As most existing works have analyzed, GCN encounters a dramatic accuracy plunge with over 2 graph convolutional layers on all tested datasets. 
In contrast, the performance decline of other compared methods is not as severe as GCN, and some of them even gain marginal performance improvement as the number of layers rises.
Nevertheless, several compared approaches still attain the highest accuracy with a 2-layer architecture, and sometimes performance may dwindle as the numbers of layers are larger.
In general, AdaGCN which also integrates multi-hop node embeddings behaves favorably on most datasets.
Nonetheless, as we have discussed, it still suffers from indistinguishable node features from deep layers on some datasets (e.g., BlogCatalog), owing to which deep AdaGCN leads to unsatisfactory performance.
We can discover that AGNN achieves competitive performance with fewer layers, and often outperforms other models with more stacked layers.
Above all, the proposed model maintains accuracy at a high level with more layers, and a suitable multi-layer AGNN is helpful to exploring representative high-order node features.
As for AGNN w/o Adaboost, although it does not always outperform other models, it succeeds in lessening the negative influence of over-smoothing compared with other models and performs satisfactorily on all tested datasets.
We also visualize the performance trends of compared methods in Figure \ref{Deeplayers} with more layers (32 and 64 layers), which intuitively shows the ability of compared models to overcome over-smoothing.
AGNN generally performs the best with deeper layers.
We find that AGNN also gains marginal improvement or keeps stable with 32/64 layers, which indicates that it gets rid of over-smoothing.
Generally, AGNN with no more than 20 layers can achieve the optimal accuracy, as recorded in Table \ref{Performance}.
In conclusion, these experimental results point out that the proposed AGNN has a powerful ability to mine underlying node embeddings with a deep network architecture.

\begin{figure}[!tbp]
    \centering
    \includegraphics[width=0.48\textwidth]{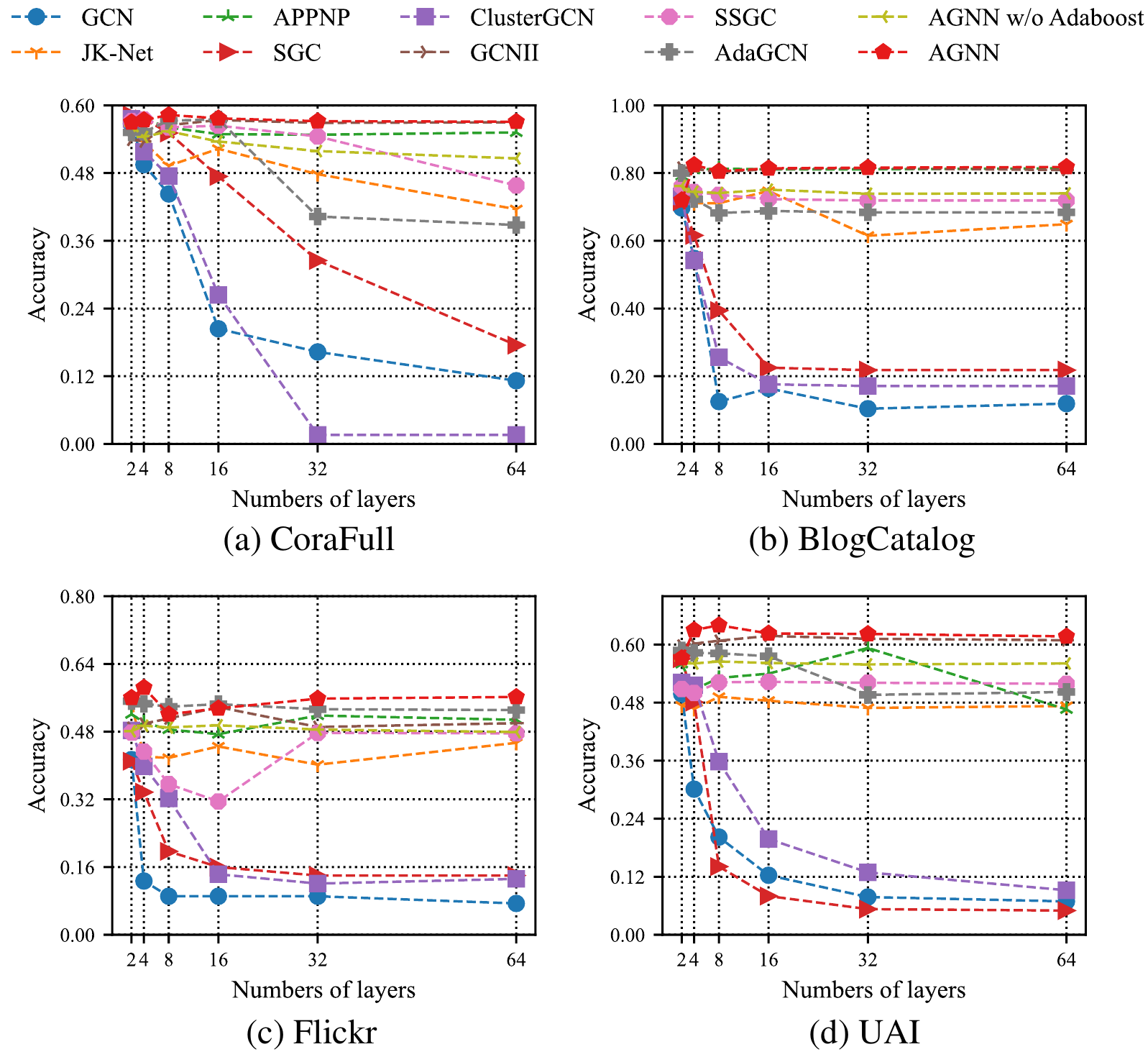}
    \caption{Performance of baselines and the proposed AGNN with 2/4/8/16/32/64 layers.}
    \label{Deeplayers}
\end{figure}

\subsubsection{Weak classifier weight distribution}
In this section, we explore the assigned weights of weak classifiers in the proposed method with varying numbers of layers, as shown in Figure \ref{Weights}.
The weight assignments demonstrate that shallow layers account for a significant portion of final predicted results, indicating that classification problems of most nodes can be effectively solved by extracting representations of one or two hops of neighbors.
In general, the top 4 layers (top 2 blocks) of AGNN play the most critical role in the final prediction, and the rest layers complement the prediction with more high-order information.
Figure \ref{Weights} reveals that AGNN achieves the best performance with 8 layers on both two selected datasets, indicating that multi-layer models are essential for improving accuracy via exploring remote neighbors.
Although Figure \ref{Weights} shows that AGNN with more than 8 layers is not the optimal selection, the improved Adaboost can maintain the classification accuracy of extremely deep networks by assigning tiny weights to deep layers, if most nodes have been correctly classified through shallow layers.
In a word, a multi-layer architecture often benefits the embedding learning, and AGNN attempts to leverage high-order information at the best.

\begin{figure*}[!tbp]
    \centering
    \includegraphics[width=\textwidth]{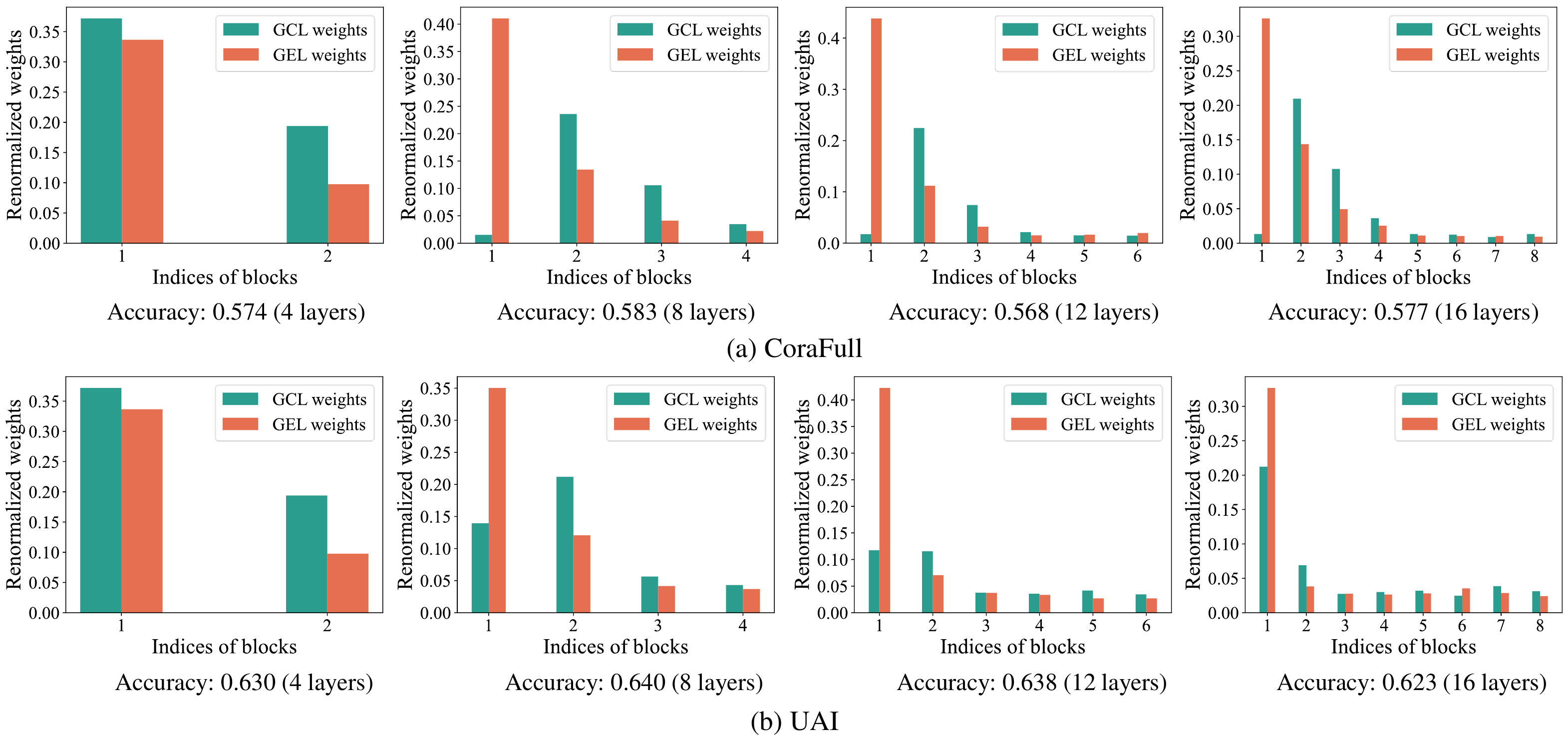}
    \caption{Weight distribution of AGNN with 4/8/12/16 layers for each weak classifier on CoraFull and UAI datasets.}
    \label{Weights}
\end{figure*}

\begin{figure*}[!htbp]
    \centering
    \includegraphics[width=\textwidth]{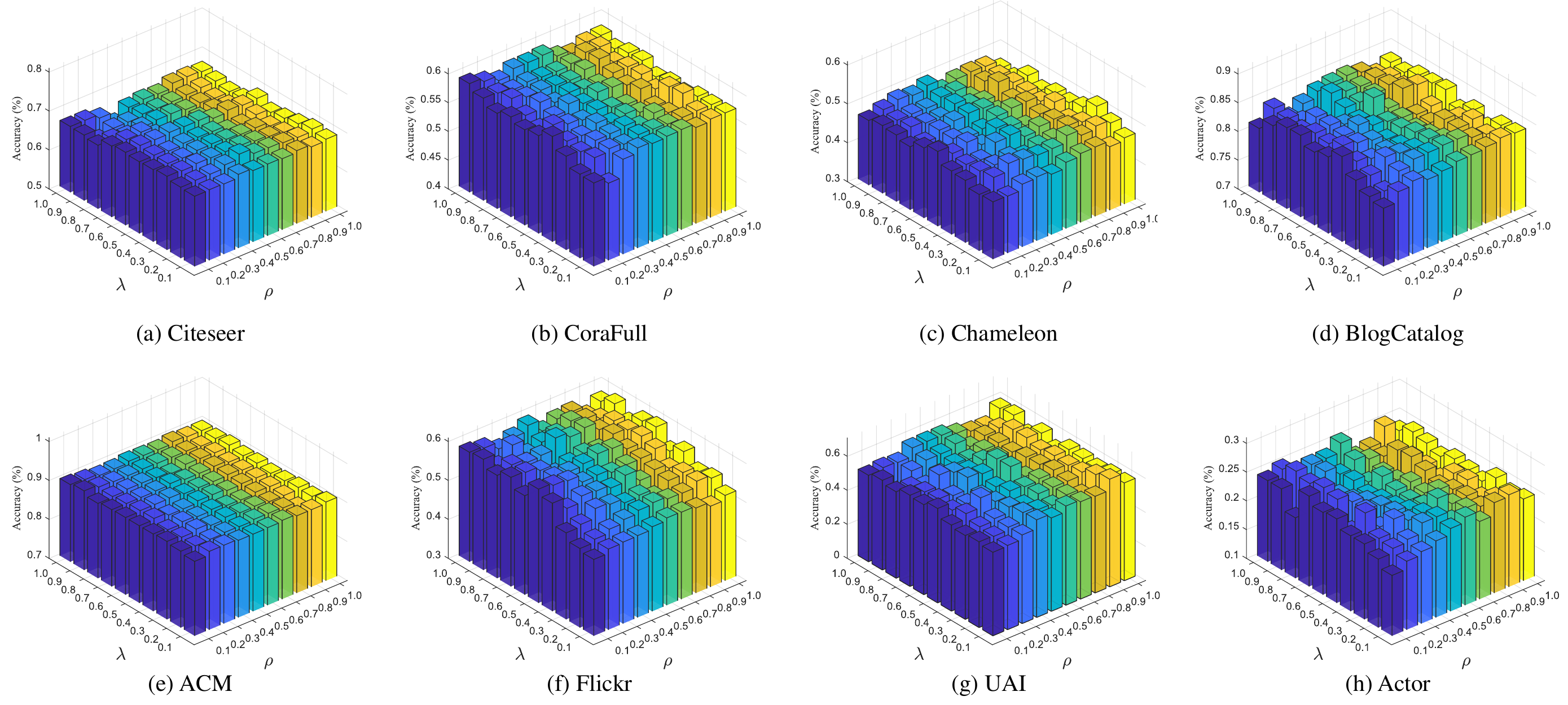}
    \caption{Parameter sensitivity of AGNN w.r.t. $\lambda$ and $\rho$ on various datasets.}
    \label{LambdaRho}
\end{figure*}

\subsubsection{Model analysis}
In this section, we further analyze the proposed model.
First, the impact of hyperparameters used in AGNN is discussed.
The accuracy changes w.r.t. $\lambda$ and $\rho$ on all datasets are demonstrated in Figure \ref{LambdaRho},
from which we find that the performance of AGNN fluctuates marginally and a suitable choice of two parameters is crucial on most datasets.
Overall, AGNN is robust to varied hyperparameters on Citeseer and ACM datasets.
Although the optimal selections of hyperparameters differ on other datasets, small values of $\lambda$ and $\rho$ often lead to undesired performance, especially on CoraFull, Chameleon, UAI and BlogCatalog datasets.
In our previous experiments, we select the optimal combination of these two hyperparameters to obtain better experimental results.

Furthermore, we validate the effectiveness of the designed activation function MSReLU in GEL, as exhibited in Table \ref{ablationAF}.
All parameter settings except those in compared activation functions are the same.
We also evaluate the performance of AGNN with identify function and ReLU function.
It is noted that ReLU only preserves non-zero entries in the matrix.
Experimental results indicate that MSReLU function succeeds in promoting classification accuracy compared with taking other functions as activation functions, attributed to the ability of making sparse outputs closer to original features. 
Sometimes, AGNN with ReLU encounters severe performance decline (e.g., Flickr and UAI datasets).
This is because that it ignores negative entries in the feature matrix, which often results in the information loss.
In a word, these observations suggest that a suitable MSReLU function benefits the learning of more accurate and robust node embeddings.

\begin{table}[!htbp]
    \centering
    \begin{tabular}{l|cccc}
    \toprule 
    Methods / Datasets             & BlogCatalog & Flickr  & UAI   & Chameleon\\ \midrule
    AGNN + IF &  0.805  & 0.568 & 0.615 & 0.450 \\
    AGNN + ST & 0.813       & 0.561   & 0.612   & 0.437\\
    AGNN + ReLU & 0.815  &  0.411 &  0.594 & 0.458 \\
    AGNN + MSReLU  & \textbf{0.824}       & \textbf{0.584}   & \textbf{0.630}   &  \textbf{0.480}\\ \bottomrule
    \end{tabular}
    \caption{Impact of Identity Function (IF), ST, ReLU and MSReLU in GEL, where $\theta = 0.02$ (ST), $\theta_{1} = 0.02$ and $\theta_{2} = 0.04$ (MSReLU). Layer numbers are fixed as 4.}\label{ablationAF}
\end{table}

\subsubsection{Convergence analysis}
Convergence curves of the proposed AGNN on BlogCatalog, Flickr, Actor and Chameleon datasets are demonstrated in Figure \ref{Convergence}.
These curves indicate that loss values of AGNN drop as the number of iterations increases and are finally convergent.
Although loss values may sometimes fluctuate, the overall trends of curves are suggestive of their convergence.
The fluctuation during training is caused by the Adaboost strategy that reassigns sample weights at each iteration.
Nevertheless, loss values are stable and converge eventually.
The figure also shows that AGNN with shallow layers generally converges more quickly than that with deep layers, due to the larger solution space caused by more trainable parameters.
Overall, AGNN with all numbers of layers leads to similar convergent points.
However, AGNN with deeper layers generally reaches lower values of cross-entropy, indicating the ability of exploring multi-hop embeddings.
It is noteworthy that AGNN with deep layers does not always correspond to better convergence, attributed to the various data distributions of different datasets.

\begin{figure*}[!tbp]
    \centering
    \includegraphics[width=\textwidth]{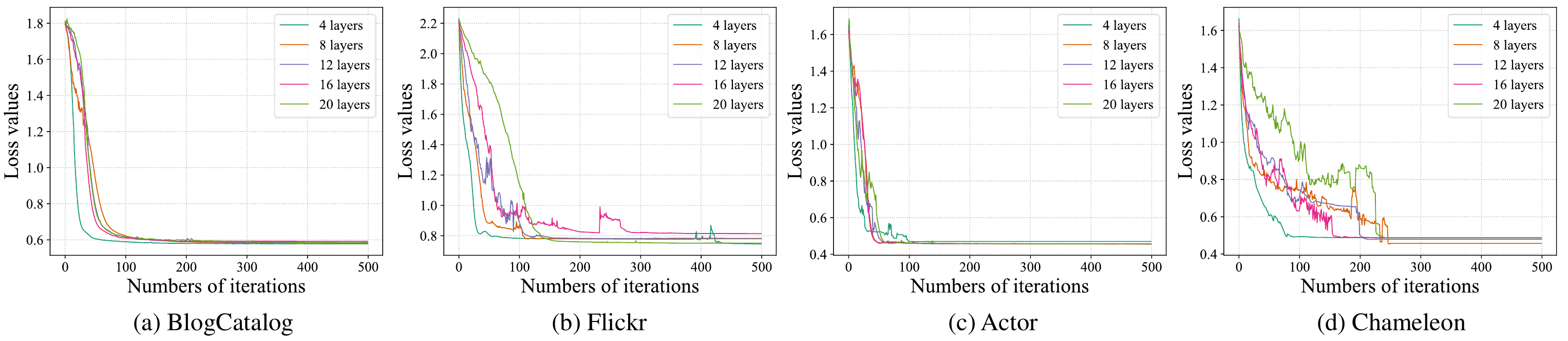}
    \caption{Training convergence curves of AGNN with varying numbers of layers ranging in $\{ 4, 8, \ldots, 20 \}$ on BlogCatalog, Flickr, Actor and Chameleon datasets.}
    \label{Convergence}
\end{figure*}

\section{Conclusion}\label{Conclusion}
In this paper, we proposed an Alternating Graph-regularized Neural Network to improve the performance of GCN in terms of semi-supervised node classification tasks, which coped with the over-smoothing issue that occurred in most GCN-based models.
We first reviewed the concept of GCN and validated that it was an approximation of a graph-regularized optimization problem.
Next, we elaborated on the proposed GEL, which was derived from another graph-regularized optimization objective formulating the transformation from the original feature space to the intermediate graph embedding space at each layer. Therefore, GEL allowed the model to carry low-order information from the input to deep layers.
Theoretically, the proposed AGNN alternately propagated node information on the basis of two graph-constrained problems.
Furthermore, an improved Adaboost strategy was leveraged to integrate hidden graph representations from all layers.
Due to more reliable and distinguishable node embeddings learned from GCL and GEL, this strategy could obtain more accurate predictions.
Extensive experiments validated that the proposed method succeeded in promoting the performance of GCN with deeper layers.
In the future, we will devote ourselves into further investigation of multi-layer GCN with techniques like attention mechanism and residual networks.

\bibliographystyle{ieeetr}
\bibliography{MachineLearning}

\end{document}